\documentclass{article}

\usepackage{microtype}
\usepackage{graphicx}
\usepackage{subfigure}
\usepackage{booktabs} 
\usepackage{algorithm}
\usepackage[algo2e,ruled,linesnumbered,lined]{algorithm2e}
\usepackage{amsmath}
\usepackage{amssymb}
\usepackage{amsthm}
\usepackage{bm}
\newtheorem{assumption}{Assumption}

\usepackage[font=small,labelfont=bf]{caption}
\newtheorem{prop}{Proposition}[section]
\newtheorem{corollary}{Corollary}[section]
\newtheorem{lemma}{Lemma}[section]
\usepackage{hyperref}

\usepackage[accepted]{icml2020}

\icmltitlerunning{Decoupled Greedy Learning of CNNs}

\begin{document}

\twocolumn[
\icmltitle{Decoupled Greedy Learning of CNNs}




\begin{icmlauthorlist}
\icmlauthor{Eugene Belilovsky}{mila}
\icmlauthor{Michael Eickenberg}{ucb}
\icmlauthor{Edouard Oyallon}{cnrs}
\end{icmlauthorlist}

\icmlaffiliation{mila}{MILA}
\icmlaffiliation{ucb}{University of California, Berkeley}
\icmlaffiliation{cnrs}{CNRS, LIP6}

\icmlcorrespondingauthor{Eugene Belilovsky}{eugene.belilovsky@mila.quebec}

\icmlkeywords{greedy training, parallel, alternatives to backpropagation}

\vskip 0.3in
]



\printAffiliationsAndNotice{} 

\begin{abstract}
A commonly cited inefficiency of neural network training by back-propagation is the \textit{update locking} problem: each layer must wait for the signal to propagate through the full network before updating. Several alternatives that can alleviate this issue have been proposed. 
In this context, we consider a simpler, but more effective, substitute that uses minimal feedback,
which we call
\emph{Decoupled Greedy Learning (DGL)}.
It is based on a greedy relaxation of the joint training objective, recently shown to be effective in the context of Convolutional Neural Networks (CNNs) on large-scale image classification. We consider an optimization of this objective that permits us to decouple the layer training, allowing for layers or modules in networks to be trained with a potentially linear parallelization in layers. With the use of a replay buffer we show this approach can be extended to asynchronous settings, where modules can operate with poossibly large communication delays. We show theoretically and empirically that this approach converges. Then, we empirically find that it can lead to better generalization than sequential greedy optimization.  We demonstrate the effectiveness of DGL against alternative approaches  on the CIFAR-10 dataset and on the large-scale ImageNet dataset. 
\end{abstract}
\begin{figure*}
    \centering
    \includegraphics[width=1.0\linewidth]{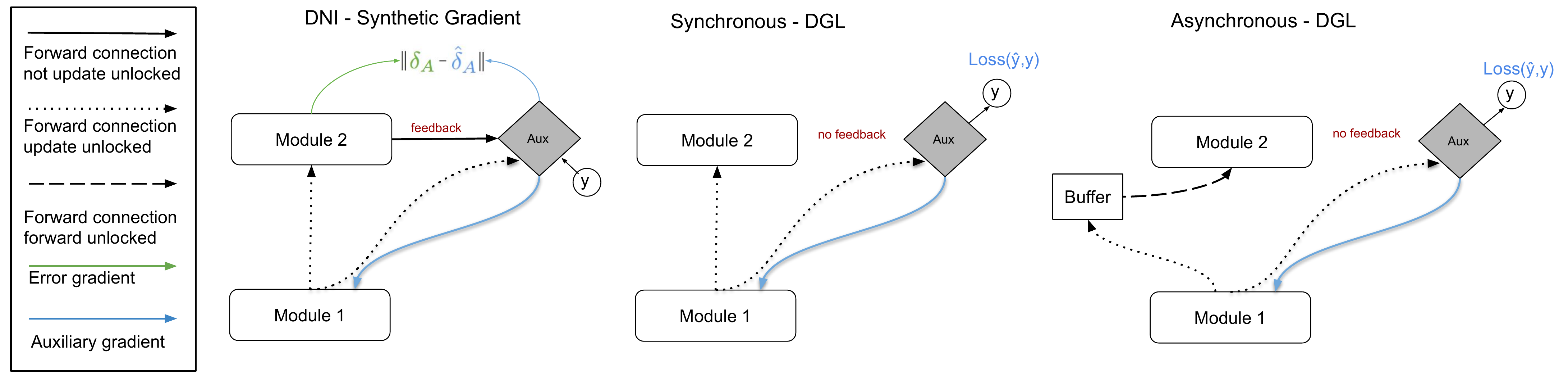}
    \caption{Comparison of DNI, Synchronous and Asynchronous DGL. Note in DGL subsequent modules do not provide feedbacks thus removing dependencies of the auxiliary network. Asynchronous DGL allows achieving forward unlocking.}
    \label{fig:dni_dgl}
    \vspace{-13pt}
\end{figure*}
\vspace{-5pt}
\section{Introduction}

Jointly training all layers using back-propagation is the standard method for learning neural networks, including the computationally intensive Convolutional Neural Networks (CNNs). 
Due to the sequential nature of gradient processing, standard back-propagation has several well-known inefficiencies that prohibit parallelization of the computations of the different constituent modules. \citet{jaderberg2016decoupled} characterize these in order of severity as the 
\textit{forward}, 
\textit{update}, and 
\textit{backward} 
\textit{locking} problems. 
Backward 
\textit{\textbf{un}locking}
would permit updates of all modules once 
forward
signals have propagated to all subsequent modules, update 
\textit{\textbf{un}locking}
would permit updates of a module \textit{before} a signal has reached all subsequent modules, and forward 
\textit{\textbf{un}locking}
would permit a module to operate asynchronously from its predecessor and dependent modules. 

Methods 
addressing backward locking to a certain degree have been proposed in \citep{ddg,Huo2018,choromanska2018beyond,direct_feedback}. However, update locking is a far more severe inefficiency. Thus \citet{jaderberg2016decoupled} and \citet{czarnecki2017} propose and analyze Decoupled Neural Interfaces (DNI), a method that
uses an auxiliary network to predict the gradient of the backward pass directly from the input. This method
unfortunately does not scale well computationally or in terms of accuracy, especially in the case of CNNs \citep{Huo2018,ddg}. Indeed, auxiliary networks must predict a  weight gradient, usually high dimensional for larger models and input sizes.

A major 
obstacle to
update unlocking 
is the heavy reliance on
the upper modules for feedback.
Several works have recently revisited the classic \citep{ivakhnenko,bengio2007greedy} approach of supervised greedy layerwise training of neural networks \citep{huang2017learning,Marquez}. In \citet{shallow}  it is shown that such an approach, which relaxes the joint learning objective, and does not require global feedback, can lead to high-performance deep CNNs on large-scale datasets. 
We will show that the greedy learning objective used in these papers can be solved with an alternative optimization algorithm, which permits decoupling the computations and 
achieves update unlocking. It can be augmented with replay buffers \citep{lin1992self} to permit 
forward unlocking which is a challenge not effectively addressed by any of the prior work. This simpler strategy can be shown to be a superior baseline for parallelizing the training across modules of a neural network.

 The paper is structured as follows. In Sec. 2 we propose an optimization procedure for a decoupled greedy learning objective that achieves \textit{update unlocking} and then extend it to an asynchronous setting (async-DGL) using a replay buffer, addressing \textit{forward unlocking}. In Sec. 3 we show that the proposed optimization procedure converges and recovers standard rates of non-convex optimization, motivating empirical observations in the subsequent experimental section. In Sec. 4 we  show that DGL can outperform competing methods in terms of scalability to larger and deeper models and stability to optimization hyperparameters and overall parallelism, allowing it to be applied to large datasets such as ImageNet. We extensively study async-DGL and find that it is robust to significant delays. We also empirically study the impact of parallelized training on convergence. Code for experiments is included in the submission.
 \vspace{-4pt}

\vspace{-2pt}
\section{Parallel Decoupled Greedy Learning}

In this section we 
formally define the greedy objective and parallel optimization which we 
study in both the synchronous and asynchronous setting. We 
mainly consider the online setting and assume a stream of samples or mini-batches denoted $\mathcal{S}\triangleq\{(x_0^t,y^t)\}_{t\leq T}$, run during $T$ iterations.
\subsection{Preliminaries}
For comparison purposes, we briefly review the update unlocking approach from  DNI \citep{jaderberg2016decoupled}. There, each network module has an associated \textit{auxiliary net} which, given the output activation of the module, predicts the gradient signal from subsequent modules: the module can thus perform an update while modules above are still forward processing. The DNI auxiliary model is trained by using true gradients provided by upper modules when they become available, requiring activation caching. This also means that the auxiliary module can become out of sync with the changing output activation distribution, often requiring slow learning rates. Due to this and the high dimensionality of the predicted gradient which scales with module size, this estimate is challenging. One may ask how well a method that entirely avoids the use of feedback from upper modules would fare given similarly-sized auxiliary networks. We will show that adapting the objective in \citep{shallow,bengio2007greedy} can also allow for update unlock and a degree of forward unlocking, with better properties.
\vspace{-10pt}
\setlength{\textfloatsep}{1pt}
    \begin{algorithm2e}\small
    \caption{Synchronous DGL}\label{algo:basic}
  \SetAlgoLined
  \DontPrintSemicolon
\KwIn{Stream $\mathcal{S}\triangleq\{(x_0^t,y^t)\}_{t\leq T}$ of samples or mini-batches.}
\textbf{Initialize} Parameters $\{\theta_j,\gamma_j\}_{j\leq J}$.\;
\For { $(x_0^t,y^t) \in \mathcal{S}$ }
{
\For {$j \in 1,..., J$}
   {$x^t_j \leftarrow f_{\theta_{j-1}}(x^t_{j-1})$.\;
Compute $\nabla_{(\gamma_j,\theta_j)}\hat{ \mathcal{L}}(y^t,x^t_j;\gamma_j,\theta_j)$.\;
$(\theta_j,\gamma_j)\leftarrow$Update parameters $(\theta_j,\gamma_j)$.
   }
 }
\end{algorithm2e}\vspace{-20pt}\begin{algorithm2e}\small
 \SetAlgoLined
  \DontPrintSemicolon
    \KwIn{Stream $\mathcal{S}\triangleq\{(x_0^t,y^t)\}_{t\leq T}$;  Distribution of the delay $p=\{p(j)\}_{j}$; Buffer size $M$.}
 \textbf{Initialize:} Buffers $\{B_j\}_{j}$; params $\{\theta_j,\gamma_j\}_{j}$.\\
 \While{\normalfont{\textbf{ training}}}{Sample $j$ in $\{1,...,J\}$ following $p$.\\
   \uIf{$j=1$}{ $ (x_{0},y)\gets \mathcal{S}$}\Else{ $(x_{j-1},y)\gets B_{j-1}$.}
   $x_j \leftarrow f_{\theta_{j-1}}(x_{j-1})$.\;
    Compute $\nabla_{(\gamma_j,\theta_j)}\hat{ \mathcal{L}}(y,x_j;\gamma_j,\theta_j)$.\;
     $(\theta_j,\gamma_j)\leftarrow$ Update parameters $(\theta_j,\gamma_j)$.\;
    \lIf{$j<J$}{
    $B_{j} \gets (x^{j},y)$.
           }}
    \caption{Asynchronous DGL with Replay\label{algo:buffer_sym}}\end{algorithm2e}
    



\vspace{-20pt}
\subsection{Optimization for Greedy Objective} \label{sec:main}

Let $\bm{X}_0$ and $Y$  be the data and labels, $\bm{X_j}$ be the output representation for module $j$. We will denote the per-module objective function $\hat{\mathcal{L}}(\bm{X}_{j}, Y;\theta_{j},\gamma_{j})$, where the parameters $\theta_j$ correspond to the module parameter (i.e. $\bm{X}_{j+1} =f_{\theta_{j}}(\bm{X}_{j})$). Here $\gamma_j$ represents parameters of a auxiliary networks used to predict the final target and compute the local objective. $\hat{\mathcal{L}}$ in our case will be the empirical risk with a cross-entropy loss. The greedy training objective is thus given recursively by defining $P_j$:
\begin{equation}
\min_{\theta_j,\gamma_j}\hat{\mathcal{L}}(\bm{X}_{j}, Y;\theta_{j},\gamma_{j})\tag{$P_j$}\label{eq:dap},
\end{equation}
where $\bm{X}_{j} =f_{\theta_{j-1}^*}(\bm{X}_{j-1})$ and $\theta_{j-1}^*$ is the minimizer of Problem  
(\(P_{j-1}\)). 
A natural way to solve the optimization problem for $J$ modules, $(P_J)$, is thus by sequentially solving the problems $\{P_{j}\}_{j\leq J}$ starting with $j=1$. This is the approach taken in e.g. \citet{Marquez,huang2017learning,bengio2007greedy,shallow}. 
Here we consider an alternative procedure for optimizing the same objective, which we refer to as Sync-DGL. It is outlined in Alg \ref{algo:basic}. In Sync-DGL individual updates of each set of parameters are performed in parallel across the different layers. Each layer processes a sample or mini-batch, then passes it to the next layer, while simultaneously performing an update based on its own local loss. Note that at line $5$ the subsequent layer can already begin computing line $4$. Therefore, this algorithm achieves update unlocking. Once $x_j^t$ has been computed,   subsequent layers can begin processing. Sync-DGL can also be seen as a generalization of the biologically plausible learning method proposed in concurrent work \citep{nokland2019training}.
Appendix~\ref{appendix:pseudo} also gives an explicit version of an equivalent multi-worker pseudo-code. Fig.~\ref{fig:dni_dgl} illustrates the decoupling compared to how samples are processed in the DNI algorithm. 

 In this work we solve the sub-problems $P_j$ by backpropagation, but we note that any iterative solver available for $P_j$ will be applicable (e.g. \cite{choromanska2018beyond}).
 Finally we emphasize that unlike the sequential solvers of (e.g. \citet{bengio2007greedy,shallow}) the distribution of inputs to each sub-problem solver changes over time, resulting in a learning dynamic whose properties have never been studied nor contrasted with sequential solvers.

\vspace{-4pt}
\subsection{Asynchronous DGL with Replay} \label{sec:asynch}

We can now extend this framework to address
\textit{forward unlocking} \citep{jaderberg2016decoupled}. DGL 
modules already do not depend on their successors for updates. 
We can further reduce dependency on the previous modules such that they can operate asynchronously. 
This is achieved via a replay buffer that is shared between adjacent modules, enabling them to reuse older samples. 
Scenarios with communication delays or substantial variations in speed between layers/modules
benefit from this. 
We study one instance of such an algorithm that uses a replay buffer of size $M$, shown in Alg. \ref{algo:buffer_sym} and illustrated in Fig.~\ref{fig:dni_dgl}. 

\if False
    \begin{algorithm2e}\small
    \caption{Synchronous DGL}\label{algo:basic}
  \SetAlgoLined
  \DontPrintSemicolon
\KwIn{Stream $\mathcal{S}\triangleq\{(x_0^t,y^t)\}_{t\leq T}$ of samples or mini-batches.}
\textbf{Initialize} Parameters $\{\theta_j,\gamma_j\}_{j\leq J}$.\;
\For { $(x_0^t,y^t) \in \mathcal{S}$ }
{
\For {$j \in 1,..., J$}
   {$x^t_j \leftarrow f_{\theta_{j-1}}(x^t_{j-1})$.\;
Compute $\nabla_{(\gamma_j,\theta_j)}\hat{ \mathcal{L}}(y^t,x^t_j;\gamma_j,\theta_j)$.\;
$(\theta_j,\gamma_j)\leftarrow$Update parameters $(\theta_j,\gamma_j)$.
   }
 }
\end{algorithm2e}\begin{algorithm2e}\small
\SetAlgoLined
  \DontPrintSemicolon
    \KwIn{Stream $\mathcal{S}\triangleq\{(x_0^t,y^t)\}_{t\leq T}$;  Distribution of the delay $p=\{p(j)\}_{j}$; Buffer size $M$.}
 \textbf{Initialize:} Buffers $\{B_j\}_{j}$; params $\{\theta_j,\gamma_j\}_{j}$.\\
\While{\normalfont{\textbf{ training}}}{Sample $j$ in $\{1,...,J\}$ following $p$.\\
   \lIf{$j=1$}{ $ (x_{0},y)\gets \mathcal{S}$}\lElse{ $(x_{j-1},y)\gets B_{j-1}$.}
   $x_j \leftarrow f_{\theta_{j-1}}(x_{j-1})$.\;
    Compute $\nabla_{(\gamma_j,\theta_j)}\hat{ \mathcal{L}}(y,x_j;\gamma_j,\theta_j)$.\;
     $(\theta_j,\gamma_j)\leftarrow$ Update parameters $(\theta_j,\gamma_j)$.\;
    \lIf{$j<J$}{
    $B_{j} \gets (x^{j},y)$.
           }}
    \caption{Asynchronous DGL with Replay\label{algo:buffer_sym}}
    
    \end{algorithm2e}
    


\fi

Our minimal distributed setting is as follows. Each worker $j$ has a buffer that it writes to and that worker $j+1$ can read from. The buffer uses a simple read/write protocol. A buffer $B_j$ lets layer $j$ write new samples. When it reaches capacity it overwrites the oldest sample. Layer $j+1$ requests samples from the buffer $B_j$. They are selected by a last-in-first-out (LIFO) rule, with precedence for the least reused samples. 
Alg. \ref{algo:buffer_sym} simulates potential delays in such a setup by the use of a probability mass function (pmf) $p(j)$ over workers, analogous to typical asynchronous settings such as \citep{leblond2017asaga}. At each iteration, a layer is chosen at random according to $p(j)$ to perform a computation. In our experiments we  limit ourselves to pmfs that are uniform over workers except for a single layer which is chosen to be selected less frequently on average. Even in the case of a uniform pmf, asynchronous behavior will naturally arise, requiring the reuse of samples.  Alg.~\ref{algo:buffer_sym} permits a controlled simulation of processing speed discrepancies and will be used over  settings of $p$ and $M$ to demonstrate that training and testing accuracy remain robust in practical regimes. Appendix \ref{appendix:pseudo}  also provides 
pseudo-code for implementation in a parallel environment.

Unlike common data-parallel asynchronous algorithms \citep{elasticSGD}, the asynchronous DGL does not rely on a master node and requires only local communication similar to recent decentralized schemes \citep{lian2017asynchronous}. Contrary to decentralized SGD, DGL nodes only need to maintain and update the parameters of their local module, permitting much larger modules.  
Combining asynchronous DGL with distributed synchronous SGD for sub-problem optimization is a promising direction. For example it can alleviate a common issue of the popular distributed synchronous SGD in deep CNNs, which is the often limiting maximum batch size \citep{im1hr}. 
\vspace{-3pt}
\subsection{Auxiliary and Primary Network Design}\label{subsec:aux}

Like DNI our procedure relies on an auxiliary network to obtain update signal, both methods thus require auxiliary network design in addition to the main CNN architecture. \citet{shallow} have shown that spatial averaging operations can be used to construct a scalable auxiliary network for the same objective as used in Sec~\ref{sec:main}. However, they did not directly consider the parallel training use case, where additional care must be taken in the design: The primary consideration is the relative speed of the auxiliary network with respect to its associated main network module. We will use primarily FLOP count in our analysis and aim to restrict our auxiliary networks to be $5\%$ of the main network. 

Although auxiliary network design might seem like an additional layer of complexity in CNN design and may require invoking slightly different architecture principles, this is not inherently prohibitive since architecture design is often related to training (e.g., the use of residuals  is originally motivated by optimization issues inherent to end-to-end backprop \citep{he2016deep}). 

Finally, we note that although we focus on the distributed learning context, this algorithm and associated theory for greedy objectives is generic and has other potential applications. For example greedy objectives have recently been used in \citep{haarnoja18a,huang2017learning} and even with a single worker DGL reduces memory.

\vspace{-2pt}
\section{Theoretical Analysis}
We now study the converge results of DGL. Since we do not rely on any approximated gradients, we can derive stronger properties than DNI \cite{czarnecki2017}, such as a rate of convergence in our non-convex setting.
To do so, we analyze  Alg. \ref{algo:basic} when the update steps are obtained from stochastic gradient methods. We show convergence guarantees \citep{bottou2018optimization} under reasonable assumptions. In standard stochastic optimization schemes, the input distribution fed to a model is fixed \citep{bottou2018optimization}. In this work, the input distribution to each module is time-varying and dependent on the convergence of the previous module. At time step $t$, for simplicity we will denote all parameters of a module (including auxiliary) as $\Theta_j^t\triangleq(\theta_j^t,\gamma_j^t)$, and samples as $Z_j^t\triangleq (X_j^t,Y^t)$, which follow the density $p_j^t(z)$.  For each auxiliary problem, we aim to prove the strongest existing guarantees \citep{bottou2018optimization,Huo2018} for the  non-convex setting despite time-varying input distributions from prior modules. Proofs 
are given in the Appendix.

Let us fix a depth $j$, such that $j>1$ and consider the converged density of the previous layer, $p^*_{j-1}(z)$. We define the following distance: $c^t_{j-1}\triangleq \int |p_{j-1}^t(z)-p_{j-1}^*(z)|\,dz$. Denoting $\ell$ the composition of the non-negative loss function and the network, we will study the expected risk $\mathcal{L}(\Theta_{j})\triangleq \mathbb{E}_{p^*_{j-1}}[\ell(Z_{j-1};\Theta_j)]$. We will now state several standard assumptions we use.
\begin{assumption}[$L$-smoothness] $\mathcal{L}$ is differentiable and its gradient is  $L$-Lipschitz.\end{assumption}
We consider the SGD scheme with learning rate $\{\eta_t\}_t$:
\begin{equation}
\hspace{-0.5em}\Theta^{t+1}_{j}=\Theta^t_j-\eta_t \nabla_{\Theta_j} \ell(Z_{j-1}^t;\Theta_j^t),
\end{equation}\vspace{-3pt}
where \(Z^t_{j-1}\sim p_{j-1}^t\).

 \begin{assumption}[Robbins-Monro conditions] The step sizes satisfy $\sum_t\eta_t=\infty$ yet $\sum_t\eta_t^2<\infty$.\end{assumption}
 
\begin{assumption}[Finite variance] There exists $G>0$ such that $\forall t,\Theta_j, \mathbb{E}_{p^t_{j-1}}\big[\Vert\nabla_{\Theta_j}\ell(Z_{j-1};\Theta_j)\Vert^2\big]\leq G$.\end{assumption}

The Assumptions 1, 2 and 3 are standard \citep{bottou2018optimization, Huo2018}, and we show in the following  that our proof of convergence leads to similar rates, up to a multiplicative constant. The following assumption is specific to our setting where we consider a time-varying distribution:
\begin{assumption}[Convergence of the previous layer] We assume that $\sum_{t} c^t_{j-1}<\infty$.\end{assumption}
\begin{lemma}Under Assumption 3 and 4, for all $\Theta_j,$ one has $ \mathbb{E}_{p^*_{j-1}}\big[\Vert\nabla_{\Theta_j}\ell(Z_{j-1};\Theta_j)\Vert^2\big]\leq G$.\end{lemma}



We are now ready to prove the core statement for the convergence results in this setting:
\begin{lemma}\label{lemma:main}Under Assumptions 1, 3 and 4, we have:
\begin{eqnarray*}
\mathbb{E}[\mathcal{L}(\Theta_j^{t+1})] & \leq & \mathbb{E}[\mathcal{L}(\Theta_j^{t})]+\frac{LG}{2}\eta_t^2\,\\
{} & {} & -\eta_t\big(\mathbb{E}[\Vert\nabla\mathcal{L}(\Theta_j^t)\Vert^2] -\sqrt{2}Gc^t_{j-1}\big).
\end{eqnarray*}\end{lemma}

The
expectation is taken over each random variable. Also, note that without the temporal dependency (i.e. $c_j^t=0$), this becomes analogous to Lemma 4.4 in \citep{bottou2018optimization}.  Naturally it follows, that
\begin{prop}
Under Assumptions 1, 2, 3 and 4, each term of the following equation converges:
\vspace{-5pt}
\begin{align}
\sum_{t=0}^T \eta_t
\mathbb{E}[\Vert\nabla\mathcal{L}(\Theta_j^t)\Vert^2] &\leq  \mathbb{E}[\mathcal{L}(\Theta^0_j)]\nonumber\\
&+G\sum_{t=0}^T\eta_t\left(\sqrt{2c_{j-1}^t}+\frac{L\eta_t}{2}\right)\nonumber.
\end{align}
\end{prop}Thus the DGL scheme converges in the sense of \citep{bottou2018optimization,Huo2018}.
 We can also  obtain the following rate:
\begin{corollary} The sequence of expected gradient norm accumulates around 0 at the following rate:
\begin{equation}
\inf_{t\leq T}\mathbb{E}[\Vert \nabla\mathcal{L}(\Theta_j^t)\Vert^2]\leq\mathcal{O}\left(\frac{\sum_{t=0}^T\sqrt{c_{j-1}^t}\eta_t}{\sum_{t=0}^T\eta_t}\right)\,.\end{equation}\end{corollary}

Thus compared to the sequential case, the parallel setting adds a delay that is controlled by  $\sqrt{c_{j-1}^t}$.



\begin{figure*}[t]
    \centering
    \includegraphics[scale=0.16]{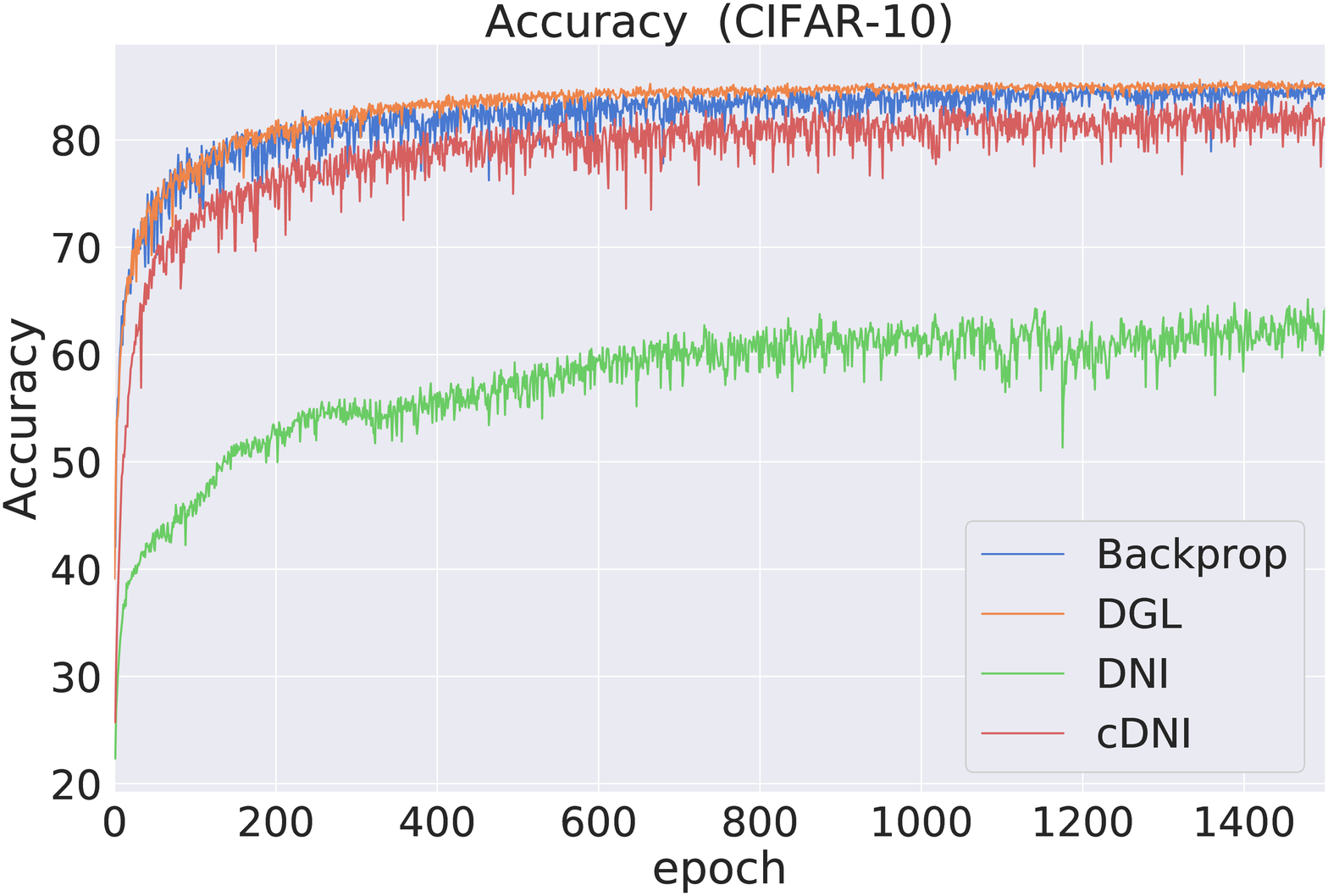}\includegraphics[scale=0.16]{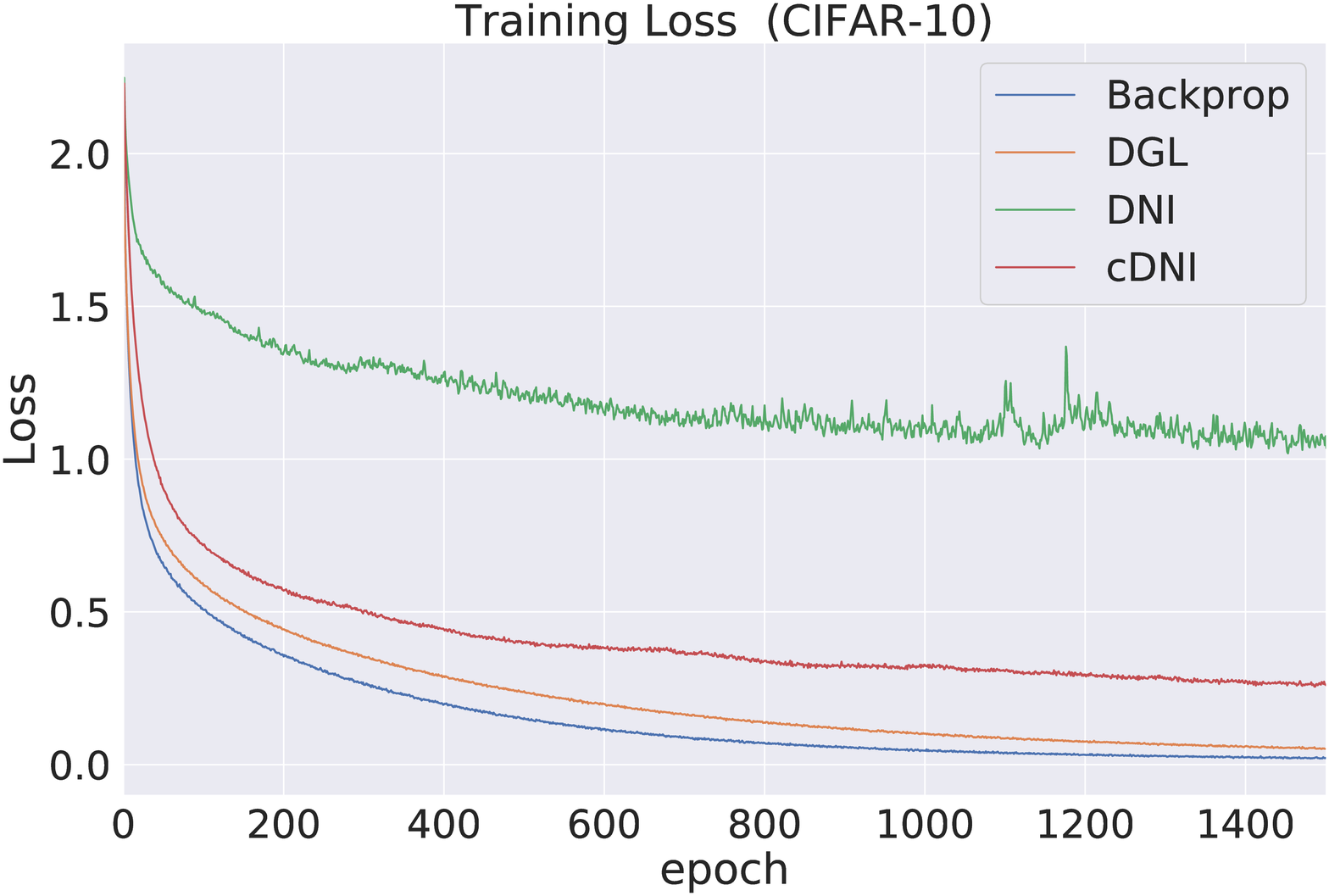}\vspace{-14pt}
    \caption{Comparison of DNI, cDNI, and DGL in terms of training loss and test accuracy for experiment from \citep{jaderberg2016decoupled}. DGL converges better than cDNI and DNI with the same auxiliary net. and generalizes better than backprop.}
    \label{fig:dni_comp}\vspace{-14pt}
\end{figure*}
\vspace{-2pt}
\section{Experiments}
We conduct experiments that empirically show that DGL optimizes the greedy objective well, showing it is favorable against recent state-of-the-art proposals for decoupling training of deep network modules. We show that unlike previous decoupled proposals it can still work on a large-scale dataset (ImageNet) and that it can, in some cases, generalize better than standard back-propagation. We then extensively evaluate the asynchronous DGL, simulating large delays. For all experiments we use architectures taken from prior works and standard optimization settings. 
\vspace{-3pt}
\subsection{Other Approaches and Auxiliary Network Designs}\label{sec:comparisons}
This section presents experiments evaluating DGL with the CIFAR-10 dataset~\citep{krizhevsky2009learning} and standard data augmentation. We  first use a  setup that permits us to compare against the DNI method and which also highlights the generality and scalability of DGL.  We then consider the design of a more efficient auxiliary network which will help to scale to the ImageNet dataset. We will also show that DGL is  effective at optimizing the greedy objective compared to a naive sequential algorithm.
\paragraph{Comparison to DNI}
We reproduce the CIFAR-10 CNN experiment described in \citep{jaderberg2016decoupled}, Appendix C.1. This experiment utilizes a 3 layer network with auxiliary networks of 2 hidden CNN layers.  We compare our reproduction to the DGL approach. Instead of the final synthetic gradient prediction for the DGL we apply a final projection to the target prediction space. Here, we follow the prescribed optimization procedure from \citep{jaderberg2016decoupled}, using Adam with a learning rate of $3\times 10^{-5}$. We run training for 1500 epochs and compare standard backprop, DNI, context DNI (cDNI) \citep{jaderberg2016decoupled} and DGL.  Results are shown in Fig. \ref{fig:dni_comp}. Details are included in the Appendix. The DGL method outperforms DNI and the cDNI by a substantial amount both in test accuracy and training loss. Also in this setting, DGL can generalize better than standard backprop and obtains a close final training loss.

We also attempted DNI with the more commonly used optimization settings for CNNs (SGD with momentum and step decay), but found that DNI would diverge when larger learning rates were used, although DGL sub-problem optimization worked effectively with common CNN optimization strategies. We also note that the prescribed experiment uses a setting where the scalability of our method is not fully exploited. Each layer of the primary network of \citep{jaderberg2016decoupled} has a pooling operation, which permits the auxiliary network to be small for synthetic gradient prediction. This however severely restricts the architecture choices in the primary network to using a pooling operation at each layer. In DGL, we can apply the pooling operations in the auxiliary network, thus permitting the auxiliary network to be negligible in cost even for layers without pooling (whereas synthetic gradient predictions  often have to be as costly as the base network). Overall, DGL is more scalable, accurate and  robust to changes in optimization hyper-parameters than DNI.

\paragraph{Auxiliary Network Design} We consider different auxiliary networks for CNNs. As a baseline we use convolutional auxiliary layers as in \citep{jaderberg2016decoupled} and \citep{shallow}. For distributed training application this approach is sub-optimal as the auxiliary network can be substantial compared to the base network, leading to poorer parallelization gains. We note however that even in those cases (that we don't study here) where the auxiliary network computation is potentially on the order of the primary network, it can still give advantages for parallelization for very deep networks and many available workers. 

The primary network architecture we use for this study is a simple CNN similar to VGG family models \citep{simonyan2014very} and those used in \citep{shallow}. It consists of 6 convolutions of size $3\times 3$, batchnorm and shape preserving padding, with $2\times2$ maxpooling at layers 1 and 3. The  width of the first layer is 128 and is doubled at each downsampling operation. The final layer does not have an auxiliary model-- it is followed by a pooling and 2-hidden layer fully connected network, for all experiments.  Two alternatives to the CNN auxiliary of \citep{shallow} are explored (Tab. \ref{tab:flop}).
    \begin{table}
    \centering
    \begin{tabular}{|c|c|c|}
    \hline&Relative FLOPS&Acc.\\\hline
     CNN-aux&   $200\%$ &92.2\\\hline
    MLP-aux&     $0.7\%$ &90.6\\\hline
     MLP-SR-aux& $4.0\%$  &91.2\\\hline

    \end{tabular}
    \caption{Comparison of auxiliary networks on CIFAR. CNN-aux applied in previous work is inefficient w.r.t. the primary module. We report flop count of the aux net relative to the largest module. MLP-aux and MLP-SR-aux applied after spatial averaging operations are far more effective with min. acc. loss. }
    \label{tab:flop}
    \end{table}

The baseline auxiliary strategy based on  \citep{shallow} and \citep{jaderberg2016decoupled} applies 2 CNN layers followed by a $2\times 2$ averaging and projection, denoted as \textit{CNN-aux}. 
First, we explore  a direct application of the spatial averaging to $2\times2$ output shape (regardless of the  resolution) followed by a 3-layer MLP (of constant width). This is denoted \textit{MLP-aux} and drastically reduces the FLOP count with minimal accuracy loss compared to \textit{CNN-aux}. Finally, we study a staged spatial resolution, first reducing the spatial resolution by 4$\times$ (and total size 16$\times$), then applying 3 $1\times 1$ convolutions followed by a reduction to $2\times 2$ and a 3 layer MLP, that we denote as \textit{MLP-SR-aux}. These latter two strategies that leverage the spatial averaging produce auxiliary networks that are less than $5\%$ of the FLOP count of the primary network even for large spatial resolutions as in real world image datasets. We will show that MLP-SR-aux is still effective even for the large-scale ImageNet dataset. We note that these more effective auxiliary models are not easily applicable in the case of DNI's gradient prediction. \paragraph{ Sequential vs. Parallel Optimization of Greedy Objective}
We briefly compare the sequential optimization of the greedy objective \citep{shallow,bengio2007greedy} to the DGL (Alg.  \ref{algo:basic}). We use a 6 layer CIFAR-10 network with an MLP-SR-aux auxiliary model. In parallel we train the layers together for 50 epochs and in the sequential training we train each layer for 50 epochs before moving to the subsequent one. Thus the difference to DGL lies only in the input received at each layer (fully converged previous layer versus not fully converged previous layer). The rest of the optimization settings are identical. Fig.~\ref{fig:greedy_v_parallel} shows comparisons of the learning curves for sequential training and DGL at layer 4 (layer 1 is the same for both as the input representation is not varying over the training period).  DGL quickly catches up with the sequential training scheme and appears to sometimes generalize better. Like \citet{oyallon2017building}, we also visualize the dynamics of training per layer in Fig.  \ref{fig:dynamics}, which demonstrates that after just a few epochs the individual layers build a dynamic of progressive improvement with depth. 
\begin{figure}[t]
\center
\includegraphics[scale=0.17]{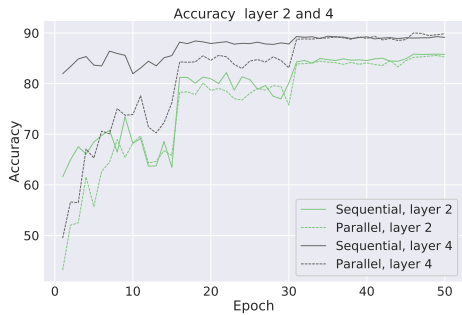}\vspace{-9pt}
    \caption{Comparison of sequential and parallel training. Parallel  catches up rapidly to sequential.}\label{fig:greedy_v_parallel}
\end{figure}
\begin{figure}
\center
    \includegraphics[scale=0.17]{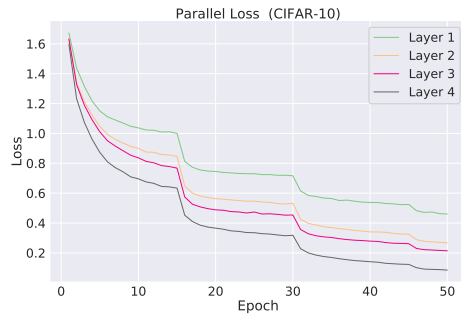}
    \vspace{-13pt}
    \caption{Per layer-loss on CIFAR: after few epochs,  the layers build a dynamic of progressive improvement in depth.}\vspace{2pt}
    \label{fig:dynamics}

\end{figure}

\paragraph{Multi-Layer modules}
We have so far mainly considered the setting of layer-wise decoupling. This approach however can easily be applied to generic modules. Indeed, approaches such as  DNI \citep{jaderberg2016decoupled} often consider decoupling entire multi-layer modules. Furthermore the propositions for backward unlocking \citep{ddg,Huo2018} also rely on and report they can often only decouple 100 layer networks into 2 or 4 blocks before observing optimization issues or performance losses and require that the number of parallel modules be much lower than the network depth for the theoretical guarantees to hold. As in those cases, using  multi-layer decoupled modules can improve performance and is natural in the case of deeper networks.  We now use such a multi-layer approach to directly compare to the backward unlocking of \citep{ddg} and then subsequently we will apply this on deep networks for ImageNet. From here on we will denote $K$ the number of total modules a network is split into.


\paragraph{Comparison to DDG} \citet{ddg} propose a solution to the backward locking (less efficient than solving update-locking, see discussion in Sec~\ref{sec:rel}). 
\begin{table}
    \center\begin{tabular}{|c|c|c|}
        \hline
           Backprop &  DDG & DGL  \\\hline
           93.53 &  93.41 & $93.5\pm0.1$ \\\hline
    \end{tabular}
    \caption{ResNet-110($K=2$) for Backprop and DDG method  from \citep{ddg}. DGL is run for 3 trials to compute variance. They give the same acc. with DGL being update unlocked, DDG only backward unlocked. DNI is reported to not work in this setting \citep{ddg}.}\vspace{2pt}
    \label{tab:ddg_comp}
\end{table}
We show that even in this situation the DGL method can provide a strong  baseline for work on backward unlocking. We take the experimental setup from \citep{ddg}, which considers a ResNet-110 parallelized into $K=2$ blocks. 
We use the auxiliary network MLP-SR-aux which has less than $0.1\%$ the FLOP count of the primary network. We use the exact optimization and network split points as in \citep{ddg}. 
 
To assess variance in  CIFAR-10 accuracy, we perform 3 trials. Tab. \ref{tab:ddg_comp} shows that the accuracy is the same across the DDG method, backprop, and our approach. DGL achieves better parallelization because it is update unlocked. We use the parallel implementation provided by \citep{ddg} to obtain a direct wall clock time comparison. We note that there are multiple considerations for comparing speed across these methods (see Appendix ~\ref{appendix:speed}). 

\paragraph{Wall Time Comparison} We compare to the parallel implementation of \citep{ddg} using the same communication protocols and run on the same hardware. We find for $K=2,4$ GPU gives a $~5\%, 18\%$ respectively speedup over DDG. With DDG $K=4$ giving approximately $2.3\times$ speedup over standard backprop on same hardware (close to results from \citep{ddg}).

\subsection{Large-scale Experiments}\label{sec:imagenet}

\begin{figure*}[t]
\begin{minipage}{\textwidth}
\begin{minipage}[b]{0.49\textwidth}
\centering
\includegraphics[scale=0.18]{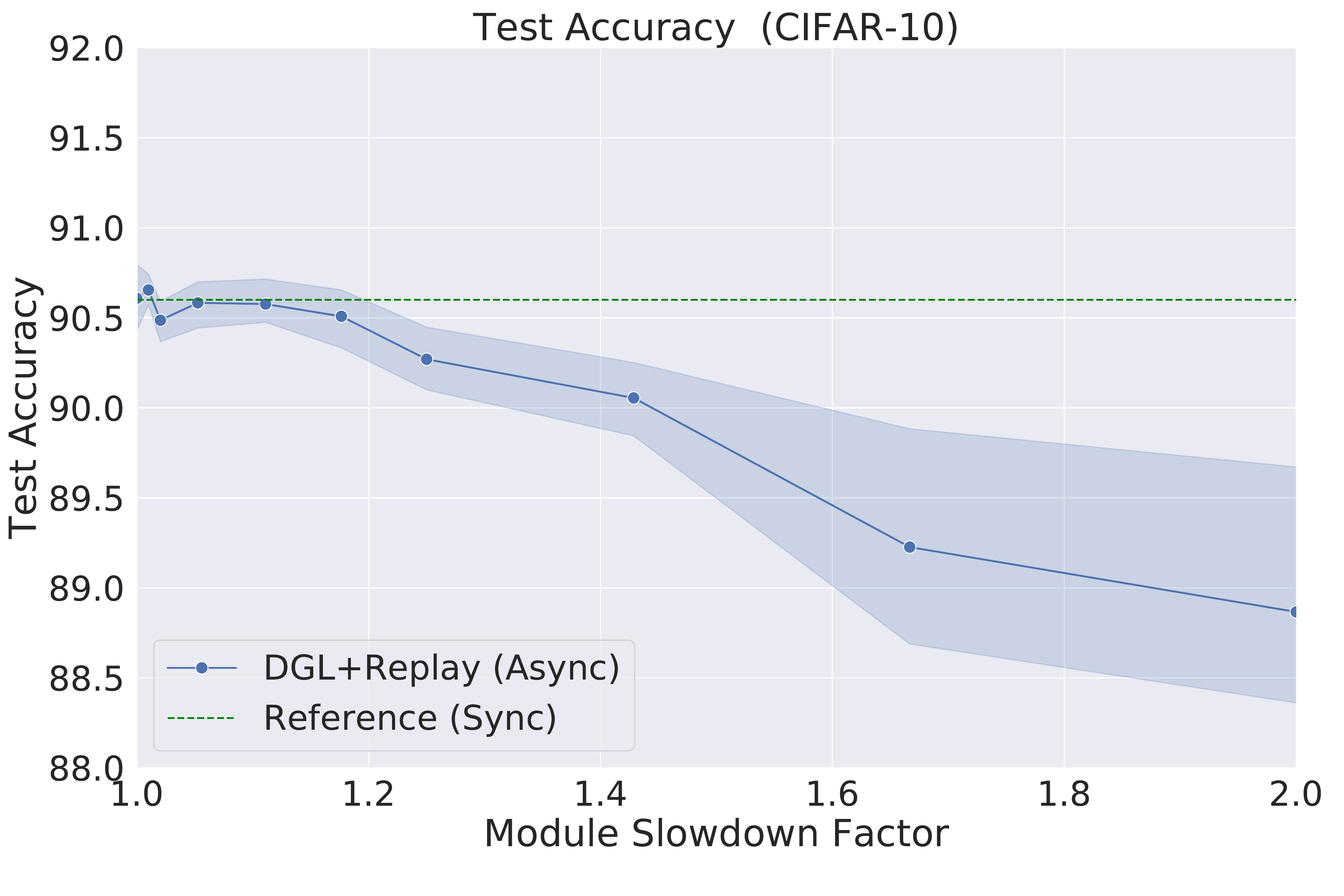}\vspace{-15pt}
\caption{Evaluation of Async DGL. A single layer is slowed down on average over others, with negligible losses of accuracy at even substantial delays.\label{fig:buffer}}
\end{minipage}
\hfill
\begin{minipage}[b]{0.48\textwidth}
\resizebox{.95\textwidth}{!}{\begin{tabular}{c|c|c}
Model (training method) & Top-1 & Top-5 \\\hline
VGG-13 (DGL per Layer, $K=10$) & 64.4 & 85.8 \\\hline
VGG-13 (DGL $K=4$) & \textbf{67.8}  & \textbf{88.0}\\\hline
VGG-13 (backprop) & 66.6 & 87.5 \\\hline\hline
VGG-19 (DGL $K=4$) & 69.2  & 89.0 \\\hline
VGG-19 (DGL $K=2$) & \textbf{70.8}  & \textbf{90.2} \\\hline
VGG-19 (backprop) & 69.7 & 89.7 \\\hline\hline
ResNet-152 (DGL $K=2$) & \textbf{74.5} &  92.0 \\\hline
ResNet-152 (backprop) & 74.4 &\textbf{92.1}\\\hline
\end{tabular}\vspace{-15pt}
}
\captionsetup{type=table} 
  \caption{ImageNet results using training schedule of \citep{Xiao2019} for DGL and standard e2e backprop. DGL with VGG and ResNet obtains similar or better accuracies, while enabling parallelization and reduced memory. \label{tab:imagenet_results}\vspace{-15pt}}
\end{minipage}
\end{minipage}
\end{figure*}
Existing methods considering update or backward locking have not been evaluated on large image datasets as they are often unstable or already show large losses in accuracy on smaller datasets. Here we study the optimization of several well-known architectures, mainly the VGG family \citep{simonyan2014very} and the ResNet \citep{he2016deep}, with DG on the ImageNet dataset. 
In all our experiments we use the MLP-SR-aux auxiliary net which scales well from the smaller CIFAR-10  to the larger ImageNet. The final module has no auxiliary network. For all optimization of auxiliary problems and for end-to-end optimization of reference models we use the shortened optimization schedule prescribed in \citep{Xiao2019}.  Results are shown in Tab.~\ref{tab:imagenet_results}. We see that for all the models DGL can perform as well and sometimes better than the end-to-end trained models, while permitting parallel training. In all these cases the auxiliary networks are neglibile (see Appendix Table ~\ref{tab:flop_im} for more details). For the VGG-13 architecture  we also evaluate the case where the model is trained layer by layer ($K=10$). Although here performance is slightly degraded, we find it is suprisingly high given that no backward communication is performed. We conjecture that improved auxiliary models and combinations with methods such as \citep{Huo2018} to allow feedback on top of the local model, may further improve performance. 
Also for the settings with larger potential parallelization, slower but more performant auxiliary models could potentially be considered as well.

The synchronous DGL has also favorable memory usage compared to DDG and to the DNI method, DNI requiring to store larger activations and DDG having high memory compared to the base network even for few splits \citep{Huo2018}. Although not our focus, the single worker version of DGL has favorable memory usage compared to standard backprop training. For example, the ResNet-152 DGL $K=2$ setting  can fit $38\%$ more samples on a single 16GB GPU than the  standard end-to-end training.

\subsection{Asynchronous DGL with Replay}

\begin{figure}
\begin{center}\vspace{-3pt}
\includegraphics[scale=0.24]{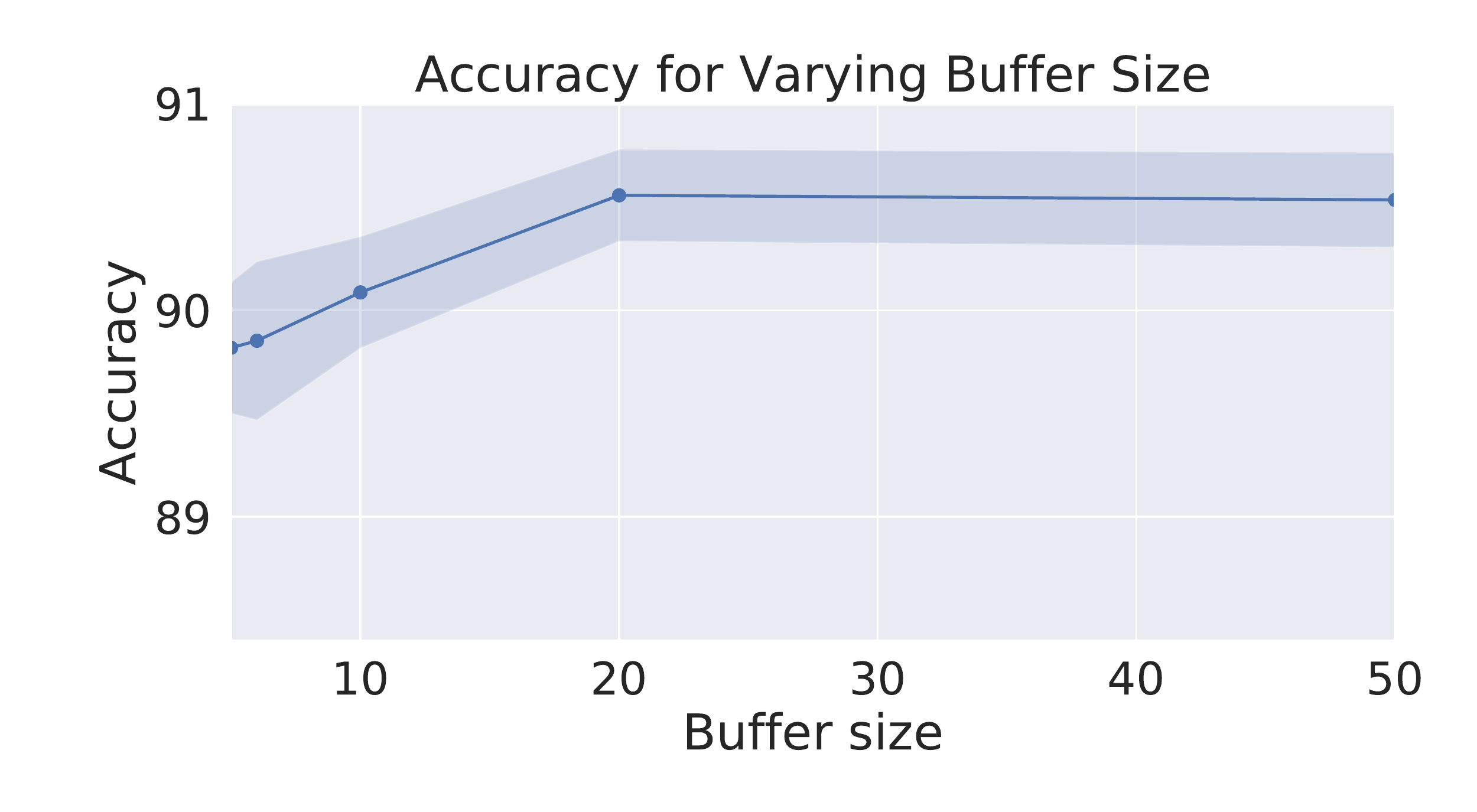}\vspace{-17pt}
\end{center}
\caption{Buffer size vs. Acc. for Async DGL. Smaller buffers produce only small loss in acc.}\label{fig:buffer2}\end{figure} 

We now study the effectiveness of
Alg. \ref{algo:buffer_sym} w.r.t delays. We use a 5 layer CIFAR-10 network with the MLP-aux and with all other architecture and optimization settings as in the auxiliary network experiments of Sec.~\ref{sec:comparisons}.  
Each layer is equipped with a buffer of size $M$.
At each iteration, a layer is chosen according to the pmf $p(j)$, and a batch selected from buffer $B_{j-1}$.
One layer is slowed down by decreasing its selection probability in the pmf $p(j)$ by a factor $S$.
We evaluate
different slowdown factors (up to $S=2.0$). Accuracy versus $S$ is
shown in Fig. \ref{fig:buffer}. For this experiment we use a buffer of size $M=50$. We run separate experiments with the slowdown applied at each layer of the network as well as 3 
random
seeds for each of these settings (thus 18 experiments per data point). We show the evaluations for 10 values of $S$.
To ensure 
a fair comparison we stop updating layers once they have completed 50 epochs, ensuring and identical number of gradient updates for all layers in all experiments.

In practice one could continue updating until all layers are trained. In Fig.~\ref{fig:buffer} we compare to the synchronous case. First, observe that the accuracy of the synchronous algorithm is maintained in the setting where $S=1.0$ and the pmf is uniform. Note that even this is a non-trivial case, as it will mean that layers inherently have random delays (as compared to  Alg. ~\ref{algo:basic}). Secondly, observe that accuracy is maintained until approximately $1.2\times$ and accuracy losses after that the difference remains small.
Note that even case $S=2.0$ is somewhat drastic: for 50 epochs, the slowed-down layer is only on epoch 25 while those following it are at epoch 50. 


We now consider the
performance with respect to the buffer size. Results are shown in Fig.~\ref{fig:buffer2}. For this experiment we set $S=1.2\times$. Observe that even a tiny buffer size can yield only a slight loss in performance accuracy. Building on this demonstration there are multiple directions to improve Async DGL with replay.
For example improving the efficiency of the buffer~\cite{oyallon2018compressing}, by including data augmentation in feature space \citep{verma2018manifold}, mixing samples in batches, or improved batch sampling, among others. 
\vspace{-8pt}
\section{Related work}\label{sec:rel}

To the best of our knowledge \citep{jaderberg2016decoupled} is the the first to directly consider the update or forward locking problems in deep feed-forward networks. Other works \citep{Huo2018,ddg} study the backward locking problem. Furthermore, a number of backpropagation alternatives \citep{choromanska2018beyond,dtargetProp,direct_feedback} 
can address backward locking. However, update locking is a more severe inefficiency. Consider the case where each layer's forward processing time is $T_F$ and is equal across a network of $L$ layers.  
Given that the backward pass is a constant multiple in time of the forward, in the most ideal case the backward unlocking will still only scale as $\mathcal{O}(LT_F)$ with $L$ parallel nodes, while update unlocking could scale as $\mathcal{O}(T_F)$.

One class of alternatives to standard back-propagation aims to avoid its biologically implausible aspects, most notably the weight transport problem \citep{bartunov2018assessing,direct_feedback,feedback_alignment,dtargetProp}. Some of these methods \citep{dtargetProp,direct_feedback} can also achieve backward unlocking as they permit all parameters to be updated at the same time, but only once the signal has propagated to the top layer. 
However, they do not solve the update or forward locking problems. Target propagation uses a local auxiliary network as in our approach, for propagating backward optimal activations computed from the layer above.  Feedback alignment replaces the symmetric weights of the backward pass with random weights. Direct feedback alignment extends the idea of feedback alignment passing errors from the top to all layers, potentially enabling simultaneous updates. These approaches have also not been shown to scale to large datasets \citep{bartunov2018assessing}, obtaining only $17.5\%$ top-5 accuracy on ImageNet (reference model achieving $59.8\%$). On the other hand, greedy learning  has been shown to work well on this task \citep{shallow}. We also note concurrent work in the context of biologically plausible models by \citep{nokland2019training} which improves on results from \cite{mostafa2018deep}, showing an approach similar to a specific instantiation of the synchronous version of DGL. This work however does not consider the applications to unlocking nor asynchronous training and cannot currently scale to ImageNet.

Another line of  related work inspired by optimization methods such as \emph{Alternating Direction Method of Multipliers (ADMM)} \citep{admmNet, carreira2014distributed, choromanska2018beyond}  use auxiliary variables to break the optimization into sub-problems. These approaches are fundamentally different from ours as they optimize for the joint training objective, the auxiliary variables providing a link between a layer and its successive layers, whereas we consider a different objective where a layer 
has no dependence on its successors. None of these methods can achieve update or forward unlocking. However, some \citep{choromanska2018beyond} are able to have simultaneous weight updates (backward unlocked). Another issue with ADMM methods is that most of the existing approaches except for \citep{choromanska2018beyond} require standard (``batch'') gradient descent and are thus difficult to scale.
They also often involve an inner minimization problem and have thus not been demonstrated to work on large-scale datasets. Furthermore, none of these have been combined with CNNs.   

Distributed optimization based on data parallelism is a popular area  in machine learning beyond deep learning models and often studied in the convex setting \citep{JMLR:v19:17-650}. For deep network optimization the predominant method
is distributed synchronous SGD \citep{im1hr} and variants, as well as asynchronous  \citep{elasticSGD} variants. Our work is closer to a form of model parallelism rather than data parallelism, and can be easily combined with many data parallel methods (e.g. distributed synchronous SGD). Finally, recent proposals for ``pipelining'' \citep{huang2018gpipe} consider systems level approaches to optimize latency times. These methods do not address the update, forward, locking problems\citep{jaderberg2016decoupled} which are algorithmic constraints of the learning objective and backpropagation. Pipelining can be seen as a  attempting to work around these restrictions, with the fundamental limitations remaining. Removing and reducing update, backward, forward locking would simplify the design and efficiency of such systems-level machinery. Finally, tangential to our work \citet{lee2015deeply} considers auxiliary objectives but with a joint learning objective, which is not capable of addressing any of the problems considered in this work.  

\vspace{-6pt}
\section{Conclusion}

We have analyzed and introduced a simple and strong baseline for parallelizing per layer and per module computations in CNN training. This work matches or exceeds state-of-the-art approaches addressing these problems and is able to scale to much larger datasets than others. Future work can develop improved auxiliary problem objectives and combinations with delayed feedback.
\subsection*{Acknowledgements}
EO acknowledges NVIDIA for its GPU donation. EB acknowledges funding from IVADO. We would like to thank John Zarka, Louis Thiry, Georgios Exarchakis, Fabian Pedregosa, Maxim Berman, Amal Rannen,  Kyle Kastner, Aaron Courville, and Nicolas Pinto for helpful discussions.
\bibliography{bib_greedy}
\bibliographystyle{icml2020}

\newpage
\appendix
\onecolumn

\section{Proofs}
\label{sec:app_proofs}
\setcounter{section}{4}
\setcounter{prop}{0}
\setcounter{lemma}{0}
\renewcommand{\thelemma}{\arabic{section}.\arabic{lemma}}
\renewcommand{\theprop}{\arabic{section}.\arabic{lemma}}

\begin{lemma}Under Assumption 3 and 4, one has: $\forall \Theta_j, \mathbb{E}_{p^*_{j-1}}\big[\Vert\nabla_{\Theta_j}\ell(Z_{j-1};\Theta_j)\Vert^2\big]\leq G$.\end{lemma}
\begin{proof}
First of all, observe that under Assumption 4 and via Fubini's theorem:
\begin{equation}
    \sum_t c_{j-1}^t = \sum_t \int |p_{j-1}^t(z)-p_{j-1}^*(z)|\,dz=\int\sum_t |p_{j-1}^t(z)-p_{j-1}^*(z)|\,dz <\infty
\end{equation}
thus, $\sum_t |p_j^t-p_j^*|$ is convergent a.s. and $|p_j^t-p_j^*|\rightarrow 0$ a.s as well. From Fatou's lemma, one has:

\begin{align}
    \int p^*_{j-1}(z)\Vert\nabla_{\Theta_j} \ell(z;\Theta_j)\Vert^2 \,dz&= \int \lim\inf_t p^t_{j-1}(z)\Vert\nabla_{\Theta_j} \ell(z;\Theta_j)\Vert^2 \,dz\\ 
    &\leq \lim\inf_t \int  p^t_{j-1}(z)\Vert\nabla_{\Theta_j} \ell(z;\Theta_j)\Vert^2 \,dz\leq G
\end{align}

\if False
then, observe that this implies that:
\begin{equation}
    p_j^t(z)\leq p_j^*(z)+(p_j^t(z)-p_j^*(z)) \leq p_j^*(z)+ |p_j^t(z)-p_j^*(z)|\leq p_j^*(z)+\sum_t |p_j^t(z)-p_j^*(z)|
\end{equation}
thus, $\sup_t p_j^t$ is integrable because the right term is integrable as well. Then, observe that:

\[ \Vert \nabla\ell_{j,t} \Vert |p^*_j(z)-p^t_j(z)| =  \Vert \nabla\ell_{j,t} \Vert 1_{p^*_j(z)< p^t_j(z)}|p^*_j(z)-p^t_j(z)|+ \Vert \nabla\ell_{j,t} \Vert 1_{p^*_j(z)\geq  p^t_j(z)}|p^*_j(z)-p^t_j(z)|\]

Then, the left term is bounded because:

\begin{equation}
    \Vert \nabla\ell_{j,t} \Vert 1_{p^*_j(z)< p^t_j(z)}|p^*_j(z)-p^t_j(z)|\leq \Vert \nabla\ell_{j,t} \Vert 1_{p^*_j(z)< p^t_j(z)}p^t_j(z) \leq \Vert \nabla\ell_{j,t} \Vert 1_{p^*_j(z)< p^t_j(z)}\sup_t p^t_j(z)
\end{equation}

\[\int \sum_t   1_{p^*_j(z)\geq  p^t_j(z)}(p^*_j(z)-p^t_j(z))\,dz\leq \int \sum_t   1_{p^*_j(z)=  p^t_j(z)}|p^*_j(z)-p^t_j(z)|\,dz<\infty\]

In particular:
\[\int \sum_t   1_{p^*_j(z)\geq  p^t_j(z)}p^*_j(z)\,dz<\infty\]

It implies that $\sum_t   1_{p^*_j(z)\geq  p^t_j(z)}$ is almost surely finite, and, a.s. $\forall z,\exists t_0, p^*_j(z)\leq p^{t_0}_j(z)$. In particular this implies that a.s.:
\[\forall z, p^*_j(z)\leq \sup_t p_j^t(z)\]

\fi

\end{proof}

\begin{lemma}Under Assumptions 1, 3 and 4, we have:
\begin{align*}
\mathbb{E}[\mathcal{L}(\Theta_j^{t+1})]\leq\mathbb{E}[\mathcal{L}(\Theta_j^{t})]-\eta_t\big(\mathbb{E}[\Vert\nabla\mathcal{L}(\Theta_j^t)\Vert^2]-\sqrt{2}Gc^t_{j-1}\big)+\frac{LG}{2}\eta_t^2\,,
\end{align*}
Observe that the expectation is taken over each random variable.
\end{lemma}

\begin{proof}
By $L$-smoothness:
\begin{align}
\mathcal{L}(\Theta_j^{t+1})\leq\mathcal{L}(\Theta_j^t)&+\nabla\mathcal{L}(\Theta_j^t)^{T}(\Theta^{t+1}_{j}-\Theta_{j}^t)+\frac{L}{2}\Vert\Theta_{j}^{t+1}-\Theta_{j}^t\Vert^{2}
\end{align}
Substituting $\Theta_{j}^{t+1}-\Theta_{j}^{t}$ on the right:
\begin{align}
\mathcal{L}(\Theta_j^{t+1})&\leq\mathcal{L}(\Theta_j^t)-\eta_t\nabla\mathcal{L}(\Theta_j^t)^{T} \nabla_{\Theta_j} \ell(Z_{j-1}^t;\Theta_j^t)+\frac{L\eta_t^2}{2}\Vert\nabla_{\Theta_j} \ell(Z_{j-1}^t;\Theta_j^t)\Vert^{2}
\end{align}

Taking the expectation w.r.t. $Z_{j-1}^t$ which has a density $p_{j-1}^t$, we get:
\begin{align*}
\mathbb{E}_{p^t_{j-1}}[\mathcal{L}(\Theta_j^{t+1})]&\leq\mathcal{L}(\Theta_j^t)-\eta_t\nabla\mathcal{L}(\Theta_j^t)^{T}\mathbb{E}_{p_{j-1}^t}[\nabla_{\Theta_j} \ell(Z_{j-1}^t;\Theta_j^t)]
+\frac{L\eta_t^2}{2}\mathbb{E}_{p_{j-1}^t}\big[\Vert\nabla_{\Theta_j} \ell(Z_{j-1}^t;\Theta_j^t)\Vert^{2}\big]
\end{align*}

From Assumption 3, we obtain that:
\begin{equation}
    \frac{L\eta_t^2}{2}\mathbb{E}_{p_{j-1}^t}\big[\Vert \nabla_{\Theta_j} \ell(Z_{j-1}^t;\Theta_j^t)\Vert^{2}\big]\leq \frac{L\eta_t^2G}{2}
\end{equation}

Then, as a side computation, observe that:

\begin{align}
\Vert \mathbb{E}_{p_{j-1}^t}\big[\nabla_{\Theta_j} \ell(Z_{j-1}^t;\Theta_j^t)\big]-\nabla\mathcal{L}(\Theta_j^t)\Vert&=\Vert \int \nabla \ell(z,\Theta_j^t)p^t_{j-1}(z)\,d z -\int \nabla \ell(z,\Theta_j^t)p^*_{j-1}(z)\,d z\Vert\\
&\leq   \int \Vert\nabla \ell(z,\Theta_j^t)\Vert ~ |p^t_{j-1}(z)-p^*_{j-1}(z)|\,d z\\
&=  \int \big(\Vert\nabla \ell(z,\Theta_j^t)\Vert\sqrt{|p^t_{j-1}(z)-p^*_{j-1}(z)|}\big) ~ \sqrt{|p^t_{j-1}(z)-p^*_{j-1}(z)|}\,d z\\
\end{align}
Let us apply the Cauchy-Swchartz inequality, we obtain:

\begin{align}
\Vert \mathbb{E}_{p_{j-1}^t}\big[\nabla_{\Theta_j} \ell(Z_{j-1}^t;\Theta_j^t)\big]-\nabla\mathcal{L}(\Theta_j^t)\Vert&\leq \sqrt{\int \Vert\nabla \ell(z,\Theta_j^t)\Vert^2 |p^t_{j-1}(z)-p^*_{j-1}(z)|\,d z} \sqrt{\int |p^t_{j-1}(z)-p^*_{j-1}(z)|\,d z }\\
&= \sqrt{\int \Vert\nabla \ell(z,\Theta_j^t)\Vert^2 |p^t_{j-1}(z)-p^*_{j-1}(z)|\,d z} \sqrt{c_{j-1}^t}
\end{align}

Then, observe that:

\begin{align}
\int \Vert\nabla \ell(z,\Theta_j^t)\Vert^2 |p^t_{j-1}(z)-p^*_{j-1}(z)|\,d z &\leq  \int \Vert\nabla \ell(z,\Theta_j^t)\Vert^2 \big(p^t_{j-1}(z)+p^*_{j-1}(z)\big)\,d z\\
&=\mathbb{E}_{p^t_{j-1}}[ \Vert\nabla \ell(Z_{j-1},\Theta_j^t)\Vert^2 ]+\mathbb{E}_{p^*_{j-1}}[ \Vert\nabla \ell(Z_{j-1},\Theta_j^t)\Vert^2 ]\\
&\leq 2G
\end{align}

The last inequality follows from Lemma 4.1 and Assumption 3.

Then, using again Cauchy-Schwartz inequality:
\begin{align}
\bigg|\Vert\nabla\mathcal{L}(\Theta_j^t)\Vert^2- \nabla\mathcal{L}(\Theta_j^t)^{T}\mathbb{E}_{p_{j-1}^t}[\nabla_{\Theta_j} \ell(Z_{j-1}^t;\Theta_j^t)]\bigg|
&=\bigg|\nabla\mathcal{L}(\Theta_j^t)^T\big(\nabla\mathcal{L}(\Theta_j^t)-\mathbb{E}_{p_{j-1}^t}[\nabla_{\Theta_j} \ell(Z_{j-1}^t;\Theta_j^t)]\big)\bigg|\\
&\leq \Vert \nabla \mathcal{L}(\Theta^t_j)\Vert~\Vert \mathbb{E}_{p_{j-1}^t}\big[\nabla_{\Theta_j} \ell(Z_{j-1}^t;\Theta_j^t)\big]-\nabla\mathcal{L}(\Theta_j^t)\Vert \\
&\leq \Vert \nabla \mathcal{L}(\Theta^t_j)\Vert \sqrt{2Gc^t_{j-1}}
\end{align}

Then, taking the expectation leads to

\begin{align}
\big|\mathbb{E}\bigg[\Vert\nabla\mathcal{L}(\Theta_j^t)\Vert^2- \nabla\mathcal{L}(\Theta_j^t)^{T}\mathbb{E}_{p_{j-1}^t}[\nabla_{\Theta_j} \ell(Z_{j-1}^t;\Theta_j^t)]\bigg]\big|&\leq \mathbb{E}[\bigg|\Vert\nabla\mathcal{L}(\Theta_j^t)\Vert^2- \nabla\mathcal{L}(\Theta_j^t)^{T}\mathbb{E}_{p_{j-1}^t}[\nabla_{\Theta_j} \ell(Z_{j-1}^t;\Theta_j^t)]\bigg|]\\
&\leq \mathbb{E}[\Vert\nabla \mathcal{L}(\Theta^t_j)\Vert] \sqrt{2Gc^t_{j-1}}\\
&\leq \sqrt{ \mathbb{E}[\Vert\nabla \mathcal{L}(\Theta^t_j)\Vert^2] }\sqrt{2Gc^t_{j-1}}\\
\end{align}
However, observe that by Lemma 4.1 and Jensen inequality:

\begin{align}
\Vert \nabla\mathcal{L}(\Theta_j^t)\Vert^2 =\Vert \mathbb{E}_{p^*_j}[\nabla_{\Theta_j} \ell(Z,\Theta_j^t)]\Vert^2\leq \mathbb{E}_{p^*_j}[ \Vert \nabla_{\Theta_j} \ell(Z,\Theta_j^t)\Vert^2]\leq G
\end{align}
Combining this inequality and Assumption 3, we get:
\begin{align*}
\mathbb{E}[\mathcal{L}(\Theta_j^{t+1})]\leq\mathbb{E}[\mathcal{L}(\Theta_j^{t})]-\eta_t\big(\mathbb{E}[\Vert\nabla\mathcal{L}(\Theta_j^t)\Vert^2]-\sqrt{2}Gc^t_{j-1}\big)+\frac{LG}{2}\eta_t^2\,,
\end{align*}

\end{proof}

\begin{prop}
Under Assumptions 1, 2, 3 and 4, each term of the following equation converges:
\begin{align}
\hspace{-0.2cm}
\begin{split}
\sum_{t=0}^T \eta_t
\mathbb{E}[\Vert\nabla\mathcal{L}(\Theta_j^t)\Vert^2] \leq  \mathbb{E}[\mathcal{L}(\Theta^0_j)] +G\sum_{t=0}^T\eta_t(\sqrt{2c_{j-1}^t}+\frac{L\eta_t}{2})
\end{split}
\end{align}
\end{prop}
\begin{proof} Applying Lemma \ref{lemma:main} for $t=0,...,T-1$, we obtain (observe the telescoping sum), for our non-negative loss:
\begin{align}\sum_{t=0}^T \eta_t \mathbb{E}[\Vert\nabla\mathcal{L}(\Theta_j^t)\Vert^2]& \leq  \mathbb{E}[\mathcal{L}(\Theta^0_j)]-\mathbb{E}[\mathcal{L}(\Theta^{T+1}_j)] +\sqrt 2G\sum_{t=0}^T\sqrt{c_j^t}\eta_t +\frac{LG}{2}\sum_{t=0}^T \eta_t^2\\
&\leq \mathbb{E}[\mathcal{L}(\Theta^0_j)] +\sqrt 2G\sum_{t=0}^T\sqrt{c_j^t}\eta_t +\frac{LG}{2}\sum_{t=0}^T \eta_t^2\\
\end{align}

Yet, $\sum \sqrt{c_j^t}\eta_t$ is convergent, as $\sum c_j^t$  and $\sum_t \eta_t^2$ are convergent, thus the right term is bounded.
\end{proof}


\setcounter{section}{1}
\section{Additional Descriptions of Experiments}
Here we provide some additional details of the experiments. Code for experiments is provided along with the supplementary materials.
\paragraph{Comparisons to DNI} The comparison to DNI attempts to directly replicate the the Appendix C.1 \citep{jaderberg2016decoupled}. Although the baseline accuracies for backprop and cDNI are close to those reported in the original work, those of DNI are worse than those reported in \citep{jaderberg2016decoupled}, which could be due to minor differences in the implementation. We utilize a popular pytorch DNI implementation available and source code will be provided.

\paragraph{Auxiliary network study} We use SGD with momentum of $0.9$ and weight decay $5\times 10^{-4}$  \citep{zagoruyko2016wide} and a short schedule of 50 epochs and decay factor of $0.2$ every 15 epochs \citep{shallow}. 
\paragraph{Sequential vs Greedy optimization experiments} We use the same architecture and optimization as in the Auxiliary network study

\paragraph{Imagenet} We use the shortened optimization schedule prescribed in \citep{Xiao2019}. It consists of training for 50 epochs with mini-batch size $256$, uses SGD with momentum of 0.9,  weight decay of $10^{-4}$, and a learning rate of $0.1$ reduced by a factor 10 every 10 epochs.

\section{Detailed Discussion of Relative Speed of Competing Methods}\label{appendix:speed}
Here we describe in more detail the elements governing differences between methods such as DNI\citep{jaderberg2016decoupled}, DDG/FA\citep{Huo2018}, and the simpler DGL. We will argue that if we take the assumption that each approach runs for the same number of epochs or iterations and applies the same splits of the network then DGL is by construction faster than the other methods which rely on feedback.   
The relative speeds of these methods are governed by the following:
\begin{enumerate}
    \item  Computation besides forward and backward passes on primary network modules (e.g. auxiliary networks forward and backward passes) 
    \item Communication time of sending activations from one module to the next module
    \item Communication time of sending feedback to the previous module 
    \item Waiting time for signal to reach final module
\end{enumerate}

As discussed in the text our auxiliary modules which govern (1) for DGL are negligible thus the overhead of (1) is negligible. DNI will inherently have large auxiliary models as it must predict gradients, thus (1) will be much greater than in DGL. (2) should be of equal speed across all methods given the same implementation and hardware. (3) does not exist for the case of DGL but exists for all other cases.  
(4) applies only in the case of backward unlocking methods (DDG/FA) and does not exist for DNI or DGL as they are update unlocked.   

Thus we observe that DGL by construction is faster than the other methods. We note that for use cases in multi-GPU settings communication would need to be well optimized for use of any of these methods. Although we include a parallel implementation based on the software from \citep{ddg}, an optimized distributed implementations of the ideas presented here and related works is outside of the scope of this work.  

\subsection{Auxiliary Network Sizes and FLOP comparisons on ImageNet}
\begin{table}[]
    \centering
    \begin{tabular}{c|c|c}
       & Flops Net & Flops Aux\\\hline
     VGG-13 ($K=4$) &13 GFLOPs & 0.2  GFLOP\\\hline
     VGG-19 ($K=4$) &20 GFLOPs & 0.2 GFLOP\\\hline
     ResNet-152 ($K=2$)  &   11 GFLOP &  0.02 GFLOP\\\hline
    \end{tabular}

    \caption{ImageNet comparisons of FLOPs for auxiliary model in major models trained. Auxiliary networks are negligible.}
    \label{tab:flop_im}
\end{table}
We briefly illustrate the sizes of auxiliary networks. Lets take as an example the ImageNet experiments for VGG-13. At the first layer the output is $224\times224\times64$. The MLP-aux here would be applied after averaging to $2\times2\times 64$, and would consists of 3 fully connected layers of size $256$ ($2*2*64$) followed by a projection to $1000$ image categories. The MLP-SR-aux network used would first reduce to $56\times56\times64$ and then apply 3 layers of $1\times1$ convolutions of width 64. This is followed by reduction to $2\times2$ and 3 FC layers as in the MLP-aux network. As mentioned in Sec. ~\ref{sec:imagenet} the auxiliary networks are neglibile in size. We further illustrate this in \ref{tab:flop_im}.

\section{Additional pseudo-code}\label{appendix:pseudo}
To illustrate the parallel implementations of the Algorithms we show a different pseudocode implementation with an explicit behavior for each worker specified. The following Algorithm~\ref{_algo:basic_parallel} is equivalent to Algorithm~\ref{algo:basic} in terms of output but directly illustrates a parallel implementation. Similarly ~\ref{algo:buff_para} illustrates a parallel implementation of the algorithm described in Algorithm~\ref{algo:buffer_sym}. The probabilities used in Algorithm~\ref{algo:buff_para} are not included here as they are derived from communication and computation speed differences. Finally we illustrate the parallelism compared to backprop in \ref{fig:update_lock_greedy}

    \begin{figure*}[t]
        \centering
        \includegraphics[width=0.6\linewidth]{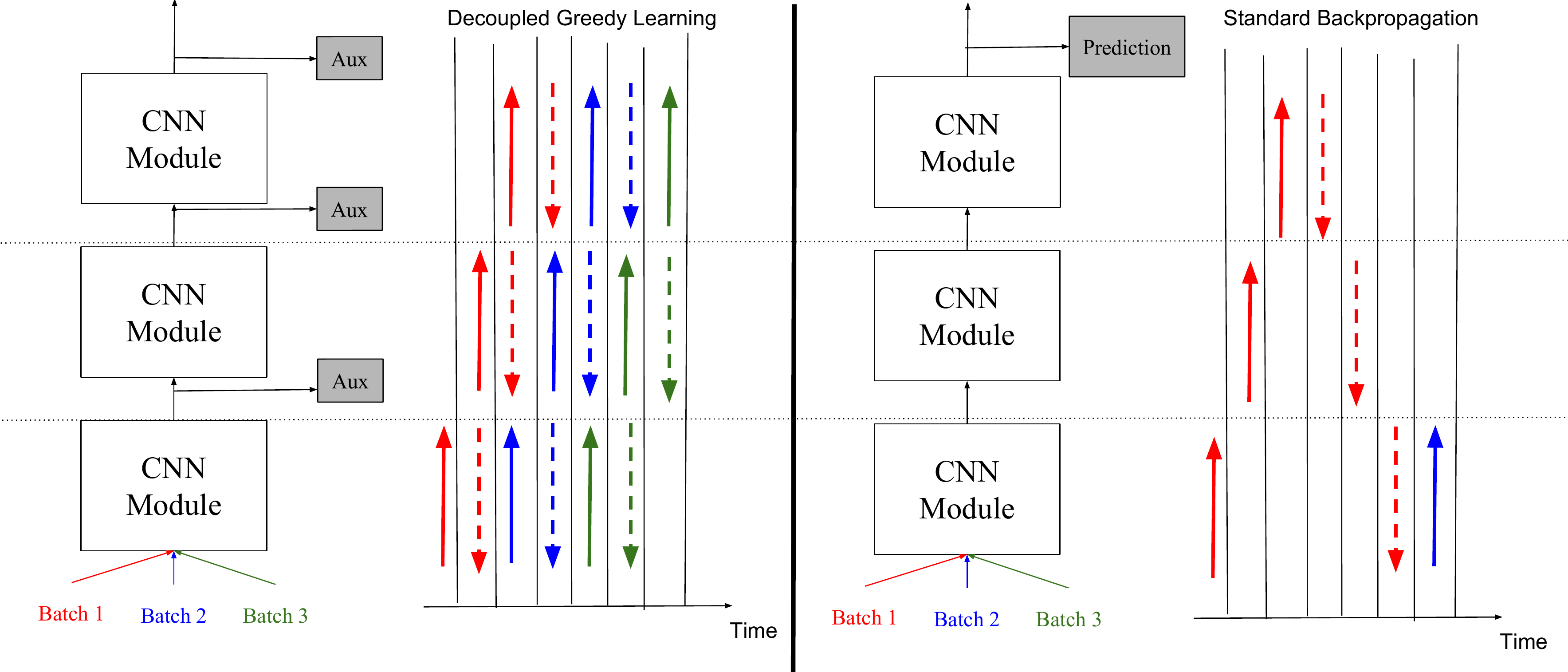}
        \caption{We illustrate the signal propagation for three mini-batches processed by standard back-propagation and with decoupled greedy learning. In each case a module can begin processing forward and backward passes as soon as possible. For illustration we assume same speed for forward and backward passes, and discount the auxiliary network computation (negligible here).} 
    \label{fig:update_lock_greedy}
    \end{figure*}
    
\begin{figure*}[h]
\begin{minipage}{\textwidth}
\begin{minipage}[b]{0.51\textwidth}
\begin{algorithm2e}[H]\small
\caption{DGL Parallel Implementation\label{_algo:basic_parallel}}
  \SetAlgoLined
  \DontPrintSemicolon
 \KwIn{Stream $\mathcal{S}\triangleq\{(x_0^t,y^t)\}_{t\leq T}$ of samples or mini-batches;}
\textbf{Initialize} Parameters $\{\theta_j,\gamma_j\}_{j\leq J}$\\
Worker 0:\;
\For {$x_0^t \in \mathcal{S}$}{
$x^t_1 \leftarrow f_{\theta^t_{0}}(x^t_{0})$\\
Send $x_0^t$ to worker $1$\;
Compute $\nabla_{(\gamma_1,\theta_1)}\hat{ \mathcal{L}}(y^t,x^t_0;\gamma^t_0,\theta^t_0)$\;
$(\theta^{t+1}_0,\gamma^{t+1}_0)\leftarrow$ Step of parameters $(\theta^t_0,\gamma^t_0)$}
Worker $j$:\;
\For{$t \in 0 ... T$}{
Wait until $x^t_{j-1}$ is available\;
$x^t_j \leftarrow f_{\theta^t_{j-1}}(x^t_{j-1})$ \\
Compute $\nabla_{(\gamma_j,\theta_j)}\hat{ \mathcal{L}}(y^t,x^t_j;\gamma^t_j,\theta^t_j)$\\
Send $x^t_{j}$ to worker $x^t_{j+1}$\\
$(\theta^{t+1}_j,\gamma^{t+1}_j)\leftarrow$ Step of parameters $(\theta^t_j,\gamma^t_j)$}
\end{algorithm2e}
\end{minipage}
\hfill
\begin{minipage}[b]{0.46\textwidth}
\begin{algorithm2e}[H]
\small\caption{DGL Async Buffer Parallel Impl.}\label{algo:buff_para}
\SetAlgoLined
  \DontPrintSemicolon
    \KwIn{Stream $\mathcal{S}\triangleq\{(x_0^t,y^t)\}_{t\leq T}$;  Distribution of the delay $p=\{p_j\}_{j}$; Buffer size $M$ }
 \textbf{Initialize:} Buffers $\{B_j\}_{j}$ with size $M$; params $\{\theta_j,\gamma_j\}_{j}$\\
 Worker $j$: \\
\While{\normalfont{\textbf{ training}}}{
    \lIf{ $ j=1$}{ $ (x_{0},y)\gets \mathcal{S}$} \lElse{ $(x_{j-1},y)\gets B_{j-1}$}\;
    $x_j \leftarrow f_{\theta_{j-1}}(x_{j-1})$\;
    Compute $\nabla_{(\gamma_j,\theta_j)}\hat{ \mathcal{L}}(y,x_j;\gamma_j,\theta_j)$\;
     $(\theta_j,\gamma_j)\leftarrow$ Step of parameters $(\theta_j,\gamma_j)$\;
    \lIf{$j<J$}{
    $B_{j} \gets (x^{j},y)$
           }}
\end{algorithm2e}
\end{minipage}
\end{minipage}
\end{figure*}

\end{document}